\documentclass[envcountsame,pdftex,runningheads]{llncs}

\usepackage{enumerate}
\usepackage{pta}
\usepackage[mathlines]{lineno}

\makeatletter                                                                                                                                       
\usepackage[bookmarks,unicode,colorlinks=true]{hyperref}
    \def\@citecolor{blue}
    \def\@urlcolor{blue}
    \def\@linkcolor{blue}

\def\orcidID#1{\smash{\href{http://orcid.org/#1}{\protect\raisebox{-1.25pt}{\protect\includegraphics{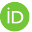}}}}}
\makeatother

\begin{document}

\title{Model-free Reinforcement Learning for Branching Markov Decision Processes\thanks{This work was supported by the Engineering and Physical Sciences Research Council through grant EP/P020909/1 and by the National Science Foundation through grant 2009022.
\includegraphics[height=8pt]{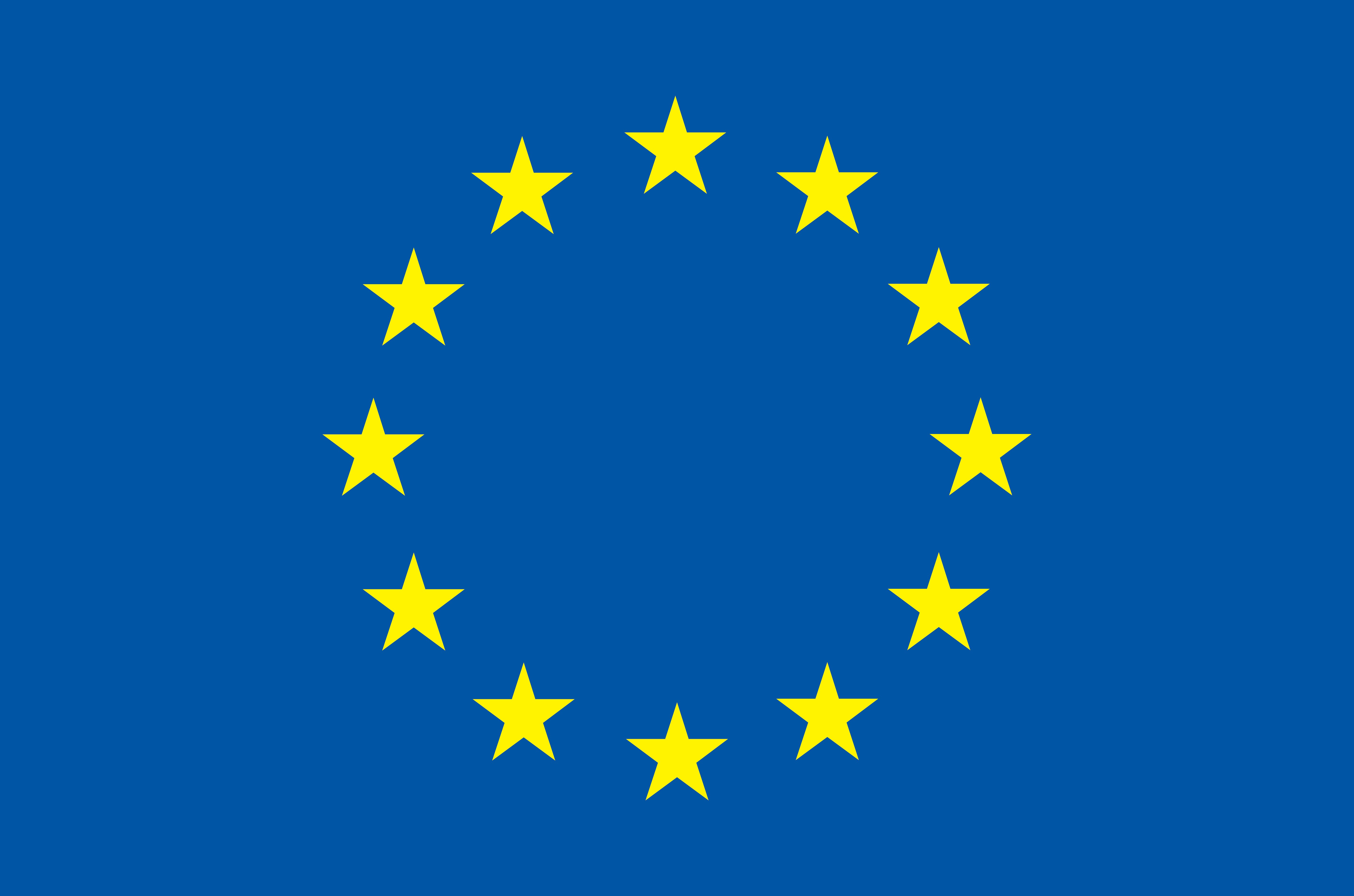} This project has received funding from the European Union’s Horizon 2020 research and innovation programme under grant agreements No 101032464 (SyGaST) and 864075 (CAESAR).
}}
\titlerunning{Reinforcement Learning for Branching Markov Decision Processes}                                                   

\author{
Ernst Moritz Hahn\inst{1}\orcidID{0000-0002-9348-7684}
\and Mateo Perez\inst{2}\orcidID{0000-0003-4220-3212} 
\and Sven Schewe\inst{3}\orcidID{0000-0002-9093-9518}
\and Fabio Somenzi\inst{2}\orcidID{0000-0002-2085-2003} 
\and Ashutosh Trivedi\inst{2}\orcidID{0000-0001-9346-0126}
\and Dominik Wojtczak\inst{3}\orcidID{0000-0001-5560-0546}
}                                                                                           
\institute{
  University of Twente, The Netherlands
  \and
  University of Colorado Boulder, USA
  \and
  University of Liverpool, UK
}                         
\authorrunning{E. M. Hahn, M. Perez, S. Schewe, F. Somenzi, A. Trivedi, and D. Wojtczak}      
\maketitle

\begin{abstract}
  We study reinforcement learning for the optimal control of Branching Markov Decision Processes (BMDPs), a natural extension of (multitype) Branching Markov Chains (BMCs). 
  The state of a (discrete-time) BMCs is a collection of entities of various types that, while spawning other entities, generate a payoff. 
  In comparison with BMCs, where the evolution of a each entity of the same type follows the same probabilistic pattern, BMDPs allow an external controller to pick from a range of options. 
  This permits us to study the best/worst behaviour of the system. 
  We generalise model-free reinforcement learning techniques to compute an optimal control strategy of an unknown BMDP in the limit.
  We present results of an implementation that demonstrate the practicality of the approach. 
\end{abstract}

\section{Introduction}
Branching Markov Chains (BMCs), also known as as Branching Processes, are natural
models of population dynamics and parallel processes.
The state of a BMC consists of entities of various types, and many entities of the
same type may coexist. 
Each entity can branch in a single step into a (possibly empty) set of entities of various types while disappearing itself.
This assumption is natural, for instance, for annual plants that reproduce only at a specific time of the year, or for bacteria, which either split or die.
An entity may spawn a copy of itself, thereby simulating the continuation of its existence.

The offspring of an entity is chosen at random among options
according to a distribution that depends on the type of the entity.
The type captures significant differences between entities.
For example, stem cells are very different from regular cells;
parallel processes may be interruptible or have different privileges.
The type may reflect characteristics of the entities such as their age or size.

Although entities coexist, the BMC model assumes that there
is no interaction between them.
Thus, how an entity reproduces and for how long it lives is
the same as if it were the only entity in the system.
This assumption greatly improves the computational complexity of the analysis
of such models and is appropriate when the population exists in an environment that has virtually unlimited resources to sustain its growth.
This is a common situation that holds when a species has just been introduced into an environment, in an early stage of an epidemic outbreak, or when running jobs in cloud computing.

BMCs have a wide range of applications in modelling various physical phenomena, such as
nuclear chain reactions, red blood cell formation, population genetics,
population migration, epidemic outbreaks, and molecular biology. 
Many examples of BMC models used in biological systems are discussed in \cite{HJV05}.

Branching Markov Decision Processes (BMDPs) extend BMCs by allowing a controller to choose the branching dynamics for each entity.
This choice is modelled as nondeterministic, instead of random.
This extension is analogous to how Markov Decision Processes (MDPs) generalise Markov chains (MCs)~\cite{Put94}.
Allowing an external controller to select a mode of branching allows us to study
the best/worst behaviour of the examined model.

As a motivating example, let us discuss a simple model of cloud computing.
A computation may be divided into tasks in order to finish it faster, as each server may have different computational power. 
Since the computation of each task depends on the previous one, the total running time is the sum of the running times of each spawned task as well as the time needed to split and merge the result of each computation into the final solution. 
As we shall see, the execution of each task is not guaranteed to be successful and is subject to random delays.  
Specifically, let us consider the following model with two different types ($T$ and $S$), and two actions ($a_1$ and $a_2$). This BMDP consists of the main task, $T$, that may be split (action $a_1$) into three smaller tasks, for simplicity assumed to be of the same type $S$, and this split and merger of the intermediate results takes 1 hour (1h). Alternatively (action $a_2$), we can execute the whole task $T$ on the main server, but it will be slow (8 hours). Task $S$ can (action $a_1$) be run on a reliable server in 1.6 hours or (action $a_2$) an unreliable one that finishes after 1 hour (irrespective of whether or not the computation is completed successfully), but with a 40\% chance we need to rerun this task due to the server crashing.
We can represent this model formally as:
\begin{align*}
    T &\stackrel{a_1}{\longrightarrow} S S S & [1\text{h}]     & &  S &\stackrel{a_1}{\longrightarrow} \epsilon & [1.6\text{h}] \\
    T &\stackrel{a_2}{\longrightarrow} \epsilon & [8\text{h}]  & & S &\stackrel{a_2}{\longrightarrow} 40\% : S \text{ or } 60\%: \epsilon & [1\text{h}]  
\end{align*}
We would like to know the infimum of the expected running time (i.e.\ the expected running time when optimal decisions are made) of task $T$. In this case the optimal control is to pick action $a_1$ first and then actions $a_1$ for all tasks S with a total running time of 5.8 hours. The expected running time when picking actions $a_2$ for $S$ instead would be $1 + 3 \cdot 1/0.6 = 6$ [hours].

Let us now assume that the execution of tasks $S$ for action $a_1$ may be interrupted with probability 30\% by a task of higher priority (type $H$).
Moreover, these $H$ tasks may be further interrupted by tasks with even higher priority (to simplify matters, again modelled by type $H$). The computation time of $T$ is prolonged by 0.1 hour for each $H$ spawned. Our model then becomes:
\begin{linenomath}
\begin{align*}
    T &\stackrel{a_1}{\longrightarrow} S S S & [1\text{h}]     & &  S &\stackrel{a_1}{\longrightarrow} 30\% : H \text{ or } 70\%: \epsilon & [1.6\text{h}] & & H &\stackrel{*}{\longrightarrow} & 30\% : HH \text{ or } \\
    T &\stackrel{a_2}{\longrightarrow} \epsilon & [8\text{h}]  & & S &\stackrel{a_2}{\longrightarrow} 40\% : S \text{ or } 60\%: \epsilon & [1\text{h}] & & & & 70\%: \epsilon \ \  [0.1\text{h}]
\end{align*}
\end{linenomath}
As we shall see, the expected total running time of $H$ can be calculated by solving the equation $x = 0.3 (x+x) + 0.1$, which gives $x = 0.25$ [hour].
So the expected running time of $S$ using action $a_1$ increases by $0.3 \cdot 0.25 = 0.075$ [hour].
This is enough for the optimal strategy of running $S$ to become $a_2$.
Note that if the probability of $H$ being interrupted is at least 50\% then the expected running time of $H$ becomes $\infty$.

When dealing with a real-life process,  it is hard to come up with a (probabilistic and controlled) model that approximates it well. 
This requires experts to analyse all possible scenarios and estimate the probability of outcomes
in response to actions based on either complex calculations or the statistical analysis of sufficient observational data. For instance, it is hard to estimate the probability of an interrupt $H$ occurring in the model above without knowing which server will run the task, its usual workload and statistics regarding the priorities of the tasks it executes.
Even if we do this estimation well, unexpected or rare events may happen that would require us to recalibrate the model as we observe the system under our control. 

Instead of building such a model explicitly first and fixing the probabilities of possible transitions in the system based on our knowledge of the system or its statistics, we advocate the use of reinforcement learning (RL) techniques~\cite{Sutton18} that were successfully applied to finding optimal control for finite-state Markov Decision Processes (MDPs). Q-learning~\cite{watkins1992q} is a well-studied model-free RL approach to compute an optimal control strategy without knowing about the model apart from its initial state and the set of actions available in each of its states. It also has the advantage that the learning process converges to the optimal control while exploiting along the way what it already knows. While the formulation of the Q-learning algorithm for BMDPs is straightforward, the proof that it works is not. This is because, unlike the MDPs with discounted rewards for which the original Q-learning algorithm was defined, our model does not have an explicit contraction in each step, nor does boundedness of the optimal values or one-step updates hold.
Similarly, one cannot generalise the result from \cite{even2003learning} that  estimates the time needed for the Q-learning algorithm to converge within $\epsilon$ of the optimal values with high probability for finite-state MDPs.

\subsection{Related work}

The simplest model of BMCs are Galton-Watson processes (\cite{WG74}),
discrete-time models where all entities are of the same type.
They date as far back as 1845 \cite{Heyde77} and were used to explain why some aristocratic
family surnames became extinct.
The generalisation of this model to multiple types of
entities was first studied in 1940s by Kolmogorov and Sevast'yanov
(\cite{KS47}). For an overview of the results known for BMCs, see
e.g. \cite{Har63} and \cite{HJV05}.
The precise computational complexity
of decision problems about the probabilities of extinction of an
arbitrary BMC was first established in \cite{EY09acm}. 
The problem of checking if a given BMC terminates almost surely 
was shown in \cite{journalsiplEsparzaGK13} to be strongly polynomial.
The probability of acceptance of a run of a BMC by a deterministic parity tree automaton was studied in \cite{ChenDK12} and shown to be computable in PSPACE and in polynomial time for probabilities 0 or 1. 
In \cite{kiefer2011probabilistic} a generalisation of the BMCs was considered that allowed for limited synchronisation of different tasks.

BMDPs, a natural generalisation of BMCs to a controlled setting, have been
studied in the OR literature (e.g., \cite{Pli77,RW82}).
Hierarchical MDPs (HMDPs) \cite{EtessamiY15} are a special case of BMDPs where there are no cycles in the offspring graph (equivalently, no cyclic dependency between types). 
BMDPs and HMDPs have found applications in manpower planning~\cite{udom2013markov}, controlled
queuing networks~\cite{jo1987optimal,brazdil-branching-networks}, management of livestock~\cite{nielsen2015markov}, and epidemic control~\cite{becker1977estimation,rao2008classical}, among
others.
The focus of these works was on optimising the expected average, or the
discounted reward over a run of the process, or optimising the population
growth rate.
In \cite{EtessamiY15} the decision problem whether the optimal probability of termination exceeds a threshold was studied: it was shown to be solvable in PSPACE and at least as hard as the square-root sum problem, but one can determine if the optimal probability is 0 or 1 in polynomial time.
In \cite{EtessamiSY20}, it was shown that the approximation of the optimal probability of extinction for BMDPs can be done in polynomial time.
The computational complexity of computing the optimal expected total cost before extinction for BMDPs follows from \cite{EtessamiWY19} and was shown there to be computable in polynomial time via a linear program formulation.
The problem of maximising the probability of reaching a state with an entity of a given type for BMDPs was studied in \cite{ETESSAMI2018355}.  
In \cite{trivedi2010timed} an extension of BMDPs with real-valued clocks and timing constraints on productions was studied.

\subsection{Summary of the results}
We show that an adaptation of the Q-learning algorithm converges almost surely to the optimal values for BMDPs under mild conditions: all costs are positive and each Q-value is selected for update independently at random. 
We have implemented the proposed algorithm in the tool \toolname{}\cite{mungojerrie} and tested its performance on small examples to demonstrate its efficiency in practice.
To the best of our knowledge, this is the first time model-free RL has been used for the analysis of BMDPs.

\section{Problem Definitions}
\subsection{Preliminaries}
We denote by $\Nat$ the set of non-negative integers, by $\Real$ the
set of reals,
by $\Rplus$ the set of positive reals, and by
$\Roplus$ the set of non-negative reals.  We let
$\Rinf = \Rplus \cup \set{\infty}$, 
and $\Roinf = \Roplus \cup \set{\infty}$.
We denote by $|X|$ the cardinality of a set
$X$ and by $X^*$ ($X^\omega$) the set of all possible finite (infinite) sequences of elements of $X$.  
Finite sequences are also called lists.

\paragraph{Vectors and Lists.}
We use $\bfx, \bfy, \bfc$ to denote vectors
and $\bfx_i$ or $\bfx(i)$ to denote its $i$-th entry. 
We let $\bfzero$ denote a vector with all entries equal to $0$; its size may vary depending on the context. Likewise $\bfone$ is a vector with all entries equal to $1$.
For vectors $\bfx, \bfy \in \Roinf^n$,
${\bfx} \leq {\bfy}$ means $x_i \leq y_i$ for every $i$,
and ${\bfx} < {\bfy}$ means ${\bfx} \leq {\bfy}$ and $x_i \neq y_i$ for some $i$. 
We also make use of the infinity norm $\|\bfx\|_\infty = \max_i |\bfx(i)|$. 

We use $\alpha, \beta, \gamma$ to denote finite lists of elements.
For a list $\alpha =  a_1, a_2, \ldots, a_k$ we write $\alpha_i$ for the $i$-th element $a_i$ of list $\alpha$ and $|\alpha|$ for its length.
For two lists $\alpha$ and $\beta$ 
we write $\alpha \cdot \beta$ for their concatenation.
The empty list is denoted by $\eps$. 

\paragraph{Probability Distributions.}
A \emph{finite discrete probability distribution} over a countable set $Q$ is
a function $\mu: Q {\to} [0, 1]$ such that $\sum_{q \in Q} \mu(q) {=} 1$ and
its support set $\supp(\mu) {=}\set{q \in Q \, | \, \mu(q) {>} 0}$ is finite.
We say that $\mu \in \DIST(Q)$ is a \emph{point distribution}
if $\mu(q) {=} 1$ for some $q \in Q$.

\paragraph{Markov Decision Processes.}
Markov decision processes (\cite{Put94}), are a well-studied
formalism for systems exhibiting nondeterministic and
probabilistic behaviour.

\begin{definition}
  A \emph{Markov decision process} (MDP) is a tuple
  $\mdp = (S, A, p, c)$ where:
  \begin{itemize}
  \item
    $S$ is the set of \emph{states};
  \item
    $A$ is the set of \emph{actions};
  \item
    $p : S \times A \to \DIST(S)$ is a partial function called the
    \emph{probabilistic transition function}; and  
  \item
    $c: S \times A \to \Real$ is the \emph{cost function}.
  \end{itemize}
\end{definition}

We say that an MDP $\mdp$ is \emph{finite} (\emph{discrete}) if both $S$
and $A$ are finite (countable).
We write $A(s)$ for the set of actions available at $s$, i.e., the
set of actions $a$ for which $p(s, a)$ is defined.
In an MDP $\mdp$, if the current state is $s$, then one of the actions in $A(s)$ is chosen nondeterministically.  If the chosen action
is $a$ then the probability of reaching state $s' \in S$ in the next step is
$p(s, a)(s')$ and the cost incurred is
$c(s,a)$.

\subsection{Branching Markov Decision Processes}

We are now ready to define (multitype) BMDPs. 

\begin{definition}
  A \emph{branching Markov decision process} (BMDP) is a tuple
  $\Bb = (\TT, A, p, c)$ where:
  \begin{itemize}
  \item
    $\TT$ is a finite set of \emph{types};
  \item
    $A$ is a finite set of \emph{actions};
  \item
    $p : \TT \times A \to \DIST(\TT^*)$ is a partial function called the
    \emph{probabilistic transition function} where every $\DIST(\cdot)$ is a finite discrete probability distribution; and 
  \item
    $c : \TT \times A \to \Rplus$ is the \emph{cost function}.
  \end{itemize}
\end{definition}
\vspace*{0.2cm}
 We write $A(q)$ for the set of actions available to an entity of type $q \in \TT$, i.e., 
 the set of actions $a$ for which $p(q, a)$ is defined.
A {\em Branching Markov Chain (BMC)} is simply a BMDP with just one action available for each type.

Let us first describe informally how BMDPs evolve.
A state of a BMDP $\Bb$ is a list of elements of
$\TT$ that we call \emph{entities}.
A BMDP starts at some initial configuration, $\alpha^0 \in \TT^*$,
and the controller picks for one of the entities one of the
actions available to an entity of its type.
In the new configuration $\alpha^1$, this one entity is replaced by the list of new
entities that it spawned.
This list is picked according to the probability distribution
$p(q,a)$ that depends both on the type of the entity, $q$, and the
action, $a$, performed on it by the controller.
The process proceeds in the same manner from $\alpha^1$,
moving to $\alpha^2$, and from there to $\alpha^3$, etc.
Once the state $\eps$ is reached, i.e., when no entities are present in the
system, the process stays in that state forever.

\begin{definition}[Semantics of BMDP]
  The semantics of a BMDP $\Bb = (\TT, A, p, c)$
  is an MDP
  $\mdp_\Bb = (\states_\Bb, \actions_\Bb, \prob_\Bb, \cost_\Bb)$
  where:
  \begin{itemize}
  \item
    $\states_\Bb = \TT^*$ is the set of states;
  \item
    $\actions_\Bb = \Nat \times A$ is the set of actions;
  \item
    $\prob_\Bb : \states_\Bb \times \actions_\Bb \to
    \DIST(\states_\Bb)$ is the probabilistic transition function such
    that, for $\alpha \in \states_\Bb$ and $(i,a) \in \actions_\Bb$, we
    have that $\prob_\Bb(\alpha, (i,a))$ is defined when
    $i \leq |\alpha|$ and $a \in A(\alpha_i)$; moreover
    \[
    \prob_\Bb(\alpha, (i,a))(\alpha_1 \ldots \alpha_{i-1} \cdot \beta \cdot \alpha_{i+1} \ldots) = p(\alpha_i, a)(\beta),
    \]
    for every $\beta \in \TT^*$ and $0$ in all other cases.
   \item
    $\cost_\Bb : \states_\Bb \times \actions_\Bb \to \Rplus$ is the
    cost function such that
    \[
    \cost_\Bb (\alpha, (i,a)) = c(\alpha_i, a).
    \]
  \end{itemize}
\end{definition}

For a given BMDP $\Bb$ and states $\alpha \in \states_\Bb$,
we denote by $\actions_\Bb(\alpha)$ the set of actions $(i,a) \in
\actions_\Bb$, for which $\prob_\Bb(\alpha, (i,a))$ is defined.

Note that our semantics of BMDPs assumes an explicit listing of all the entities in a particular order similar to \cite{EtessamiY15}. 
One could, instead, define this as a multi-set or simply a vector just counting the number of occurrences of each entity as in \cite{Pli77}. 
As argued in \cite{EtessamiY15}, all these models are equivalent to each other. 
Furthermore, we assume that the controller expands a single entity of his choice at the time rather all of them being expanded simultaneously. As argued in \cite{Wojtczak09}, that makes no difference for the optimal values of the expected total cost that we study in this paper, provided that all transitions' costs are positive.

\subsection{Strategies}
A \emph{path} of a BMDP $\Bb$ is a finite or infinite sequence
\begin{multline*}
  \pi = \alpha^0, ((i_1,a_1), \alpha^1), ((i_2,a_2), \alpha^2), ((i_3,a_3), \alpha^3),
\ldots\\
\in \states_\Bb \times {((\actions_\Bb {\times} \states_\Bb)^* \cup
(\actions_\Bb {\times} \states_\Bb)^\omega)},
\end{multline*}
consisting of the
initial state and a finite or infinite sequence of action and state pairs, such that
$\prob_\Bb(\alpha^j, (i_j,a_j))(\alpha^{j+1}) > 0$ for any $0 \leq j
\leq |\pi|$, where $|\pi|$ is the number of actions taken during path
$\pi$. ($|\pi| = \infty$ if the path is infinite.)
For a path $\pi$, we denote by $\pi_{A(j)} = (i_j, a_j)$ the
$j$-th action taken along path $\pi$, by $\pi_{S(j)}
(=\alpha^j)$ the $j$-th state visited, where $\pi_{S(0)} (=\alpha^0)$
is the initial state, and by $\pi(j) (=\alpha^0, ((i_1,a_1), \alpha^1), \ldots, ((i_j,a_j), \alpha^j))$ the first $j$ action-state pairs of $\pi$.

We call a path of infinite (finite) length a \emph{run}
(\emph{finite path}).
We write $\RUNS_\Bb$ ($\FRUNS_\Bb$) for the sets of all runs (finite
paths)  and $\RUNS_{\Bb, \alpha}$ ($\FRUNS_{\Bb, \alpha}$) for the
sets of all runs (finite paths) that start at a given initial state
$\alpha \in \states_\Bb$, i.e., paths $\pi$ with $\pi_{S(0)} =
\alpha$.
We write $\LAST(\pi)$ for the last state of a finite path $\pi$.

A \emph{strategy} in BMDP $\Bb$ is a function
$\sigma: \FRUNS_\Bb \to \DIST(\actions_\Bb)$ such that, for
all $\pi \in \FRUNS_\Bb$, $\supp(\sigma(\pi)) \subseteq
\actions_\Bb(\LAST(\pi))$.
We write $\Sigma_\Bb$ for the set of all strategies.
A strategy is called \emph{static}, if it always applies an action to the first entity in any state and for all entities of the same type in any state it picks the same action. 
A static strategy $\tau$ is essentially a function of the form $\sigma : P \to A$, i.e., for an arbitrary $\pi \in \FRUNS_\Bb$, we have
$\tau(\pi) = (1,\sigma(\LAST(\pi)_1))$ whenever $\LAST(\pi) \neq \epsilon$.

A strategy $\sigma \in \Sigma_\Bb$ and an initial state $\alpha$ induce a probability measure over the set of runs of BMDP $\Bb$ in the following way:
the basic open sets of $\RUNS_{\Bb}$ are of the form
$\pi \cdot (\actions_\Bb \times \states_\Bb)^\omega$, where
$\pi \in \FRUNS_\Bb$, and the measure of this open set is equal to
$\prod_{i=0}^{|\pi|-1} \sigma(\pi(i))(\pi_{A(i+1)})\cdot\prob_\Bb(\pi_{S(i)}, \pi_{A(i+1)})(\pi_{S(i+1)})$
if $\pi_{S(0)} = \alpha$ and equal to $0$ otherwise.
It is a classical result of measure theory that this extends to a
unique measure over all Borel subsets of $\RUNS_\Bb$ and we will
denote this measure by $\PROB^\sigma_{\Bb,\alpha}$.

Let $f : \RUNS_\Bb \to \Rinf$ be a function measurable with 
respect to $\PROB^\sigma_{\Bb,\alpha}$.
The expected value of $f$ under strategy $\sigma$ when starting at
$\alpha$ is defined as
$\eE^\sigma_{\Bb, \alpha} \set{f} = \int_{\RUNS_\Bb}\ f\
\text{d}\PROB^\sigma_{\Bb,\alpha}$ (which can be $\infty$ even if the probability that the value of $f$ is infinite is $0$).
The infimum expected value of $f$ in $\Bb$ when starting at $\alpha$
is defined as
$\minval(f) = \inf_{\sigma \in \Sigma_\Bb} \eE^\sigma_{\Bb, \alpha}
\set{f}$.
A strategy, $\opt{\sigma}$, is said to be optimal if
$\eE^{\opt{\sigma}}_{\Bb, \alpha} \set{f} = \minval(X)$
 and $\varepsilon$-optimal if
 $\eE^{\opt{\sigma}}_{\Bb, \alpha} \set{f} \leq \minval(f) + \varepsilon$.
Note that $\varepsilon$-optimal strategies always exists by definition.
We omit the subscript $\Bb$, e.g., in $\states_\Bb$, $\Sigma_\Bb$, etc.,
when the intended BMDP is clear from the context.

For a given BMDP $\Bb$ and $N \geq 0$ we define $\Total_N(\pi)$, the cumulative cost of a run $\pi$ after $N$ steps, as $\Total_N(\pi) = \sum_{i=0}^{N-1} \cost(\pi_{S(i)}, \pi_{A(i+1)}).$ For a configuration $\alpha \in \states$ and a strategy $\sigma \in \Sigma$, let $\ETotal_N(\Bb, \alpha, \sigma)$ be the \emph{$N$-step expected total cost} defined as
$\ETotal_N(\Bb, \alpha, \sigma) = \eE^\sigma_{\Bb, \alpha} \Set{\Total_N}$
and the \emph{expected total cost} be
$\ETotal_*(\Bb, \alpha, \sigma) = \lim_{N\to\infty} \ETotal_N(\Bb,\alpha,\sigma).$
This last value can potentially be $\infty$.
For each starting state $\alpha$, we compute the
\emph{optimal expected cost} over all strategies of a BMDP starting at
$\alpha$, denoted by $\ETotal_*(\Bb, \alpha)$, i.e.,
\[
\ETotal_*(\Bb, \alpha) = \inf_{\sigma \in \Sigma_{\Bb}} \ETotal(\Bb, \alpha,
\sigma).
\]
As we are going to prove in Theorem \ref{thm:lfp}.\ref{item:fixed-point} that, for any $\alpha\in\states$, we have 
\[
\ETotal_*(\Bb,\alpha) = \sum_{i=1}^{|\alpha|} \ETotal_*(\Bb,\alpha_i).
\]
This justifies focusing on this value for initial states that consist of a single entity only, as we will do in the following section.

\section{Fixed Point Equations}
\label{sec:fixed}
Following~\cite{EtessamiWY19}, we define here a
linear equation system with a minimum operator
whose {\em Least Fixed Point} solution
yields the desired optimal values for each type
of a BMDP with non-negative costs.
This system generalises the Bellman's equations for finite-state MDPs.
We use a variable $x_{q}$ for each unknown $\ETotal_*(\Bb, q)$ where $q \in \TT$. 
Let
$\bfx$ be the vector of all $x_q , \text{ where } q\in \TT$.
The system has one equation of the form $x_q= F_q(\bfx)$
for each type $q \in \TT$, defined as
\begin{linenomath}
\begin{align}
\tag{$\spadesuit$}
\label{eq:fixed-eq}
x_{q} = 
\min_{a \in A(q)} \Bigl(c(q,a) + \sum_{\alpha\in \TT^*} p(q,a)(\alpha) \sum_{i\leq |\alpha|} x_{\alpha_i} \Bigr) \enspace.
\end{align}
\end{linenomath}

\noindent We denote the system in vector form by $\bfx =
F(\bfx)$.  Given a BMDP, we can easily construct
its associated system in
linear time.
  Let ${\bfc^*} \in \Roinf^n$ denote the $n$-dimensional vector of
$\ETotal_*(\Bb, q)$'s where $n = |\TT|$. Let us  define $\bfx^0=\bfzero$,
$\bfx^{k+1}= F^{k+1}(\bfzero) = F(\bfx^{k})$, for $k \geq 0$.

\begin{theorem}\label{thm:lfp}
The following hold:
\hfill
\begin{enumerate}[(a)]
\item The map $F: \Roinf^n \to \Roinf^n$ is monotone and continuous (and so
$\bfzero \leq \bfx^k \leq \bfx^{k+1}$ for all $k \geq
0$).
\item $\bfc^*=F(\bfc^*)$.
\item For all $k \geq 0$, $\bfx^k \leq \bfc^*$.
\item \label{item:any-fixp-greater-than-c} For all $\bfc'\in \Roinf^n$, if $\bfc' =
F(\bfc')$, then $\bfc^* \leq \bfc'$.
\item $\bfc^* = \lim_{k\to\infty} \bfx^k$.
\end{enumerate}
\end{theorem}

\begin{proof}

\mbox{\ }

\begin{enumerate}[(a)]

\item All equations in the system $F(x)$ are minimum of linear functions with non-negative coefficients and constants,
and hence monotonicity and continuity are preserved.

\item \label{item:fixed-point} 
It suffices to show that once action $a$ is taken when starting with a single entity $q$ and, as a result, 
$q$ is replaced by $\alpha$ with probability $p(q,a)(\alpha)$, then the expected total cost is equal to:
\begin{align} \tag{$\clubsuit$}
    c(q,a) + \sum_{i\leq |\alpha|} \ETotal_*(\Bb, \alpha_i) \enspace. \label{eq:exp-cost}
\end{align}
This is because then the expected total cost of picking action $a$ when at $q$ is 
just a weighted sum of these expressions with weights $p(q,a)(\alpha)$ for offspring $\alpha$.
And finally, to optimise the cost, one would pick an action $a$ with the smallest such expected total cost
showing that 
\[
\ETotal_*(\Bb, q) = \min_{a \in A(q)} \Big(c(q,a) + \sum_{\alpha\in \TT^*} p(q,a)(\alpha) \sum_{i\leq |\alpha|}  \ETotal_*(\Bb, \alpha_i) \Big)
\] 
indeed holds.

Now, to show (\ref{eq:exp-cost}), consider an 
$\epsilon$-optimal strategy $\sigma_i$ for a BMDP that starts at $\alpha_i$.
It can easily be composed into a strategy $\sigma$ that starts at $\alpha$ just by 
executing $\sigma_1$ first until all descendants of $\alpha_1$ die out, 
before moving on to $\sigma_2$, etc.
If one of these strategies, $\sigma_i$, never stops executing then, due to the assumption that all costs are positive, the expected total cost when starting with $\alpha_i$ has to be infinite and 
so has to be the overall cost when starting with $\alpha$ (as all descendants of $\alpha_i$ have to die out before the overall process terminates), so (\ref{eq:exp-cost}) holds.
This shows that $c(q,a) + \sum_{i\leq |\alpha|} \ETotal_*(\Bb, x_{\alpha_i})$ can be achieved when starting at $\alpha$.
At the same time, we cannot do better because that would imply the existence of a strategy $\sigma'$ for 
one of the entities $\sigma_j$ with a better cost than its optimal cost $\ETotal_*(\Bb, \alpha_j)$. 

\item\label{item:c-least-fixed-point} 
Since $\bfx^0 = \bfzero \leq \bfc^*$ and due to (\ref{item:fixed-point}), it follows by repeated application of $F$ to both sides of this inequality
that $\bfx^k \leq F(\bfc^*) = \bfc^*$, for all $k \geq 0$.

\item\label{item:part5} Consider any fixed point $\bfc'$ of the
equation system $F(\bfx)$. We will prove that $\bfc^* \leq
\bfc'$. Let us denote by $\sigma'$ a static strategy 
that picks for each type an action with the minimum value of operator $F$ in
$\bfc'$,  i.e., for each entity $q$ we choose $\sigma'(q)=\arg
\min_{a \in A(q)} \left(c(q,a) + \sum_{\alpha\in \TT^*} p(q,a)(\alpha) \sum_{i\leq |\alpha|} \bfc'_{\alpha_i} \right)$, where we break ties lexicographically. 

We now claim that, for all $k \geq 0$, $\ETotal_k(\Bb,q,\sigma') \leq \bfc'_q$ holds. For $k=0$, this is trivial as $\ETotal_k(\Bb,q,\sigma') = 0 \leq \bfc'_q$. For $k>0$, we have that
\begin{linenomath}
\begin{align*}
\ETotal_{k}(\Bb, q, \sigma') \stackrel{(1)}{\leq} \  &c(q,\sigma'(q)) {+} \sum_{\alpha\in \TT^*} p(q,\sigma'(q))(\alpha) \sum_{i\leq |\alpha|} \ETotal_{k{-}1}(\Bb, \alpha_i, \sigma')\\
 \stackrel{(2)}{\leq}\ & c(q,\sigma'(q)) + \sum_{\alpha\in \TT^*} p(q,\sigma'(q))(\alpha) \sum_{i\leq |\alpha|} \bfc'_{\alpha_i} \\
 \stackrel{(3)}{=}\ & \min_{a \in A(q)} \Big(c(q,a) + \sum_{\alpha\in \TT^*} p(q,a)(\alpha) \sum_{i\leq |\alpha|} \bfc'_{\alpha_i}\Big) \stackrel{(4)}{=}
\bfc'_q
\end{align*}
\end{linenomath}
where (1) follows from the fact that after taking action $\sigma'(q)$ first, there are only $k-1$ steps left of the BMDP $\Bb$ that would need to be distributed among the offspring $\alpha$ of $q$ somehow. Allowing for $k-1$ steps for each of the entities $\alpha_i$ is clearly an overestimate of the actual cost. (2) follows from the inductive assumption. (3) follows from the definition of $\sigma'$. The last equality, (4), follows from the fact that $\bfc'$ is a fixed point of $F$.

Finally, for every $q \in \TT$, from the definition we have $\bfc^*_q = \ETotal_{*}(\Bb, q) \leq \ETotal_{*}(\Bb, q, \sigma') = \lim_{k\to\infty} \ETotal_{k}(\Bb, q, \sigma')$ and each element of the last sequence was just shown to be $\leq \bfc'_q$. 

\item We know that $\bfx^* = \lim_{k\to\infty}\bfx^k$ exists in $\Roinf^n$
because it is a
monotonically non-decreasing sequence (note that some entries may be infinite).
In fact we have $\bfx^* = \lim_{k \to \infty} F^{k+1}(\bfzero) =
F(\lim_{k \to \infty} F^{k}(\bfzero))$,
and thus $\bfx^*$ is
a fixed point of $F$.
So from (\ref{item:part5}) we
have $\bfc^*  \leq \bfx^*$.
At the same time, due to (\ref{item:c-least-fixed-point}), we have $\bfx^k \leq \bfc^*$
for all $k \geq 0$, so $\bfx^* = \lim_{k\to\infty}\bfx^k \leq \bfc^*$ and
thus $\lim_{k\to\infty}\bfx^k = \bfc^*$.
\end{enumerate}
\qed 
\end{proof}

\noindent The following is a simple corollary of Theorem \ref{thm:lfp}.
\begin{corollary} \label{cor:optimal-static}
In BMDPs, there exists an optimal static control strategy $\sigma^*$.
\end{corollary}
\begin{proof}
It is enough to pick as $\sigma^*$, the strategy $\sigma'$ from Theorem \ref{thm:lfp}.\ref{item:part5}, for $\bfc' = \bfc^{*}$.
We showed there that for all $k \geq 0$ and $q\in \TT$ we have $\ETotal_k(\Bb,q,\sigma^*) \leq \bfc'_q$.
So $\ETotal_*(\Bb,q,\sigma^*) = \lim_{k\to\infty} \ETotal_k(\Bb,q,\sigma^*) \leq \bfc^*_q = \ETotal_*(\Bb,q)$, 
so in fact $\ETotal_*(\Bb,q,\sigma^*) = \ETotal_*(\Bb,q)$ has to hold as clearly
$\ETotal_*(\Bb,q,\sigma^*) \geq \ETotal_*(\Bb,q)$.
\qed
\end{proof}

Note that for a BMDPs with a fixed static strategy $\sigma$
(or equivalently BMCs),
we have that $F(\bfx) = B_\sigma \bfx + \bfc_\sigma$, for some non-negative
matrix $B_\sigma \in \Roplus^{n \times n}$, 
and a positive vector $\bfc_\sigma > 0$ consisting of all one step costs $c(q,\sigma(q))$. We will refer to $F$ as $F_\sigma$ in such a case and exploit this fact later in various proofs.

We now show that $\bfc^*$ is in fact essentially a unique fixed point of $F$.

\begin{theorem}
\label{thm:unique-fix}
If $F(\bfx) = \bfx$ and $\bfx_q < \infty$ for some $q\in \TT$ then $\bfx_q = \bfc^*_q$.
\end{theorem}
\begin{proof}
By Corollary \ref{cor:optimal-static}, there exists an optimal static strategy, denoted by $\sigma^*$, which
yields the finite optimal reward vector $\bfc^*$.

We clearly have that $\bfx = F(\bfx) \leq F_{\sigma^*} (\bfx)$, because $\sigma^*$ is just one possible pick of actions for each type rather than the minimal one as in (\ref{eq:fixed-eq}).
Furthermore, 
\begin{align*}
F_{\sigma^*}(\bfx) &= B_{\sigma^*} \bfx + b_{\sigma^*} \\
& \leq B_{\sigma^*}(B_{\sigma^*} \bfx + b_{\sigma^*}) + b_{\sigma^*} \\
& = B^2_{\sigma} \bfx + (B_{\sigma^*} + 1) b_\sigma^* \\ 
& \leq \ldots \leq \lim_{k\to\infty} B^k_{\sigma^*} \bfx + \Bigl(\sum^\infty_{k=0}  B^k_{\sigma^*}\Bigr) b_{\sigma^*}.
\end{align*}

Note that $\bfc^* =  (\sum^\infty_{k=0}  B^k_{\sigma^*}) b_{\sigma^*}$,
because 
\[
\bfc^* = \lim_{k\to\infty} F^k(\bfzero) = \lim_{k\to\infty} F_{\sigma^*}^k(\bfzero) = \lim_{k\to\infty} \sum^k_{i=0} B^i_{\sigma^*} b_{\sigma^*}.
\]
Due to Theorem \ref{thm:lfp}.\ref{item:any-fixp-greater-than-c}, we know 
that $\bfc^*_q \leq \bfx_q < \infty $, so all entries in the $q$-th row of $B^k_{\sigma^*}$ have to converge to $0$ as $k\to\infty$, because otherwise the $q$-th row of $\sum^\infty_{k=0}  B^k_{\sigma^*}$ would have at least one infinite value and, as a result, the $q$-th position of $\bfc^* = (\sum^\infty_{k=0}  B^k_{\sigma^*}) b_{\sigma^*}$ would also be infinite as all entries of $b_{\sigma^*}$ are positive. Therefore,
$\lim_{k\to\infty} (B^k_{\sigma^*} \bfx)_q = 0$ and so 
\[
\bfx_q \leq (\lim_{k\to\infty} B^k_{\sigma^*} \bfx)_q + ((\sum^\infty_{k=0}  B^k_{\sigma^*}) b_{\sigma^*})_q = \bfc^*_q.
\]

The proof is now complete.\qed
\end{proof}

\section{Q-learning}
\def\lr{\lambda}

We next discuss the applicability of Q-learning to
the computation of the fixed point defined in the previous section. 

Q-learning~\cite{watkins1992q} is a well-studied model-free RL approach to
compute an optimal strategy for discounted rewards.
Q-learning computes so-called Q-values for every state-action pair.
Intuitively, once Q-learning has converged to the fixed point, $Q(s,a)$ is
the optimal reward the agent can get while performing action $a$ after starting
at $s$. 
The Q-values can be initialised arbitrarily, but ideally they
should be close to the actual values. 
Q-learning learns over a number of episodes, each consisting of a sequence
of actions with bounded length. 
An episode can terminate early if a sink-state or another non-productive state
is reached. 
Each episode starts at the designated initial state $s_0$. 
The Q-learning process moves from state to state of the MDP using one of its
available actions and accumulates rewards along the way. 
Suppose that in the $i$-th step, the process has reached state $s_i$. It then
either performs the currently (believed to be) optimal action 
(so-called {\em exploitation} option) or, with probability $\epsilon$,
picks uniformly at random one of the actions available at $s_i$ (so-called
{\em exploration} option). Either way, if $a_i$, $r_i$, and $s_{i+1}$ are the action picked, reward observed and the state the process moved to, respectively, then the Q-value is updated as follows:
\[
Q_{i+1}(s_i,a_i) = (1-\lr_i) Q_i(s_i,a_i) + \lr_i (r_i +
  \gamma \cdot \max_a Q_i(s_{i+1},a)) \enspace,
\]
where $\lr_i \in\, ]0,1[$ is the
learning rate and $\gamma \in\, ]0,1]$ is the discount factor. 
Note the model-freeness: this update does not depend on the set of transitions nor their probabilities.
For all other pairs $s,a$ we have $Q_{i+1}(s,a) = Q_i(s,a)$, i.e., they are left unchanged.
Watkins and Dayan showed the convergence of $Q$-learning \cite{watkins1992q}.
\begin{theorem}[Convergence\cite{watkins1992q}]
\label{thm:q-learning}
  For $\gamma < 1$,  bounded rewards $r_i$ and learning rates $0 \leq
  \lr_i < 1$ satisfying:
  \[ 
  \sum_{i=0}^{\infty} \lr_i  = \infty \text{ and } \sum_{i = 0} ^{\infty}
  \lr_i^2 < \infty,
  \]
  we have that $Q_i(s, a) \to Q(s, a)$ as $i \to \infty$ for all $s, a \in S
  {\times} A$ almost surely if all $(s,a)$ pairs are visited infinitely often.
\end{theorem}

However, in the total reward setting that corresponds to $Q$-learning with
discount factor $\gamma = 1$, $Q$-learning may not converge, or
converge to incorrect values.
However, it is guaranteed to work for finite-state MDPs
in the setting of undiscounted total reward with a target sink-state under the assumption
that all strategies reach that sink-state almost surely. 
The assumption that we make instead is that every transition of BMDP incurs a positive cost. This guarantees that a process that does not terminate almost surely generates an expected infinite reward in which case the Q-learning will coverage (or rather diverge) to $\infty$, so our results generalise these existing results for Q-learning.   

We adopt the Q-learning algorithm to minimise cost as follows.
Each episode starts at the designated initial state $q_0 \in P$. 
The Q-learning process moves from state to state of the BMDP using one of its
available actions and accumulates costs along the way. 
Suppose that, in the $i$-th step, the process has reached state $\alpha$. It then
selects uniformly at random one of the entities of $\alpha$, e.g., the $j$-th one, $\alpha_j$ and 
either performs the currently (believed to be) optimal action or, with probability $\epsilon$,
picks an action uniformly at random among all the actions available for $\alpha_j$. If $c$ and $\beta$ denote the observed cost and entities spawned by this action, respectively, then the Q-value of the pair $\alpha_j$, $a_i$ are updated as follows:
\[
Q_{i+1}(\alpha_j,a_i) = (1-\lr_i) Q_i(\alpha_j,a_i) + \lr_i \big(c + \sum_{i=1}^{|\beta|}
  \min_{a \in A(\beta_i)} Q_i(\beta_i,a)\big).
\]
and all other Q-values are left unchanged.
In the next section we show that Q-learning almost surely converges (diverges) to the optimal finite (respectively, infinite) value of $\bfc^*$ almost surely under rather mild conditions.

\section{Convergence of Q-Learning for BMDPs}
We show almost sure convergence of the Q-learning to the optimal values $\bfc^*$ in a number of stages.
We first focus on the case when all optimal values in $\bfc^*$ are finite. In such a case, we show a weak convergence of the expected optimal values for BMCs 
to the unique fixed-point $\bfc^*$, as defined in Section~\ref{sec:fixed}.
To establish this, we show that the expected Q-values are monotonically decreasing (increasing) if we start with Q-values $\kappa\bfc^*$ for $\kappa > 1$ ($\kappa < 1$). 
This convergence from above and below gives us convergence in expectation using the squeeze theorem.

We then establish almost sure convergence to $\bfc^*$ by proving a contraction argument,
with the extra assumption that the selection of the Q-value to update is done independently at random in each step.

In the next step, we extend this result to BMDPs, first establishing that Q-learning will almost surely converge to the \emph{region} of the Q-values less than or equal to $\bfc^*$.
We then show that, when considering the pointwise limes inferior values of the sequences of Q-values, there is no point in that region such that every $\varepsilon$-ball around it has a non-zero probability to be represented in the limes inferior.
This establishes that $\bfc^*$ is the fixed point the Q-values converge against.

Only at the very end, we show that Q-learning also converges (or rather diverges) to the optimal value even if that value happens to be infinite. 
We then turn to a type with non-finite optimal value and provide an argument for the divergence to $\infty$ of its corresponding Q-value.

We assume that all the Q-values are stored in a vector $Q$ of size $(|\TT| \cdot |A|)$.
We also use $Q(q,a)$ to refer to the entry for type $q\in \TT$ and action $a \in A(q)$.
We introduce the {\em target for $Q$ operator}, $T$, that maps a Q-values vector $Q$ to:

\[
T(Q)(q,a) = c(q,a) + \sum_{\alpha\in Q^*} p(q,a)(\alpha)
  \sum_{i=1}^{|\alpha|}
  \min_{a_i \in A(\alpha_i)} Q(\alpha_i,a_i) \enspace.
\]
We call $T$ the `target', because, when the $Q(q,a)$ value is updated, then
\[
\eE(Q_{i+1}(q,a)) = (1-\lr_i) Q_i(q,a) + \lr_i T(Q_i)(q,a)
\]
holds, whereas otherwise $Q_{i+1}(q,a) = Q_i(q,a)$.

Thus, when $Q(q,a)$ is selected for update with a chance of $p_{q,a}$, we have that
\begin{align}
\label{eq:exp-target}
\tag{$\heartsuit$}
    \eE(Q_{i+1}(q,a)) = (1-\lr_i p_{q,a}) Q_i(q,a) + \lr_i p_{q,a} T(Q_i)(q,a) \enspace.
\end{align}

\subsection{Convergence for BMCs with finite $\bfc^*$}
Since BMCs have only one action, we omit mentioning it for ease of notation.

Note that for BMCs, the target for the Q-values is a simple affine function:
\[
  T(Q)(q) = c(q) + \sum_{\alpha\in \TT^*} p(q)(\alpha)
  \sum_{i=1}^{|\alpha|}
Q(\alpha_i).
\]
And it coincides with operator $F$ as defined in Section \ref{sec:fixed}.
Therefore, due to Theorem \ref{thm:unique-fix}, $T(Q)$ has a unique fixed point which is $\bfc^*$. Moreover, $T(Q) = BQ + \bfc$, where $B$ is a non-negative matrix and $\bfc$ is a vector of one step costs $c(q)$, which are all positive. 

Naturally, applying $T$ to a non-negative vector $Q$ or multiplying it by $B$ are monotone: $Q \geq Q' \rightarrow T(Q) \geq T(Q')$ and $BQ \geq BQ'$. Also, due to the linearity of $T$, $\eE(T(Q)) = T(\eE Q)$ holds, where $Q$ is a random vector.

We now start with a lemma describing the behaviour of Q-learning for initial Q-values when they happen to be equal to $\kappa\bfc^*$ for some $\kappa \geq 1$.

\begin{lemma}
\label{lem:BP_up}
Let $Q_0 = \kappa \bfc^*$ for a scalar factor $\kappa \geq 1$. 
Then the following holds for all $i \in \mathbb N$, 
\[
\bfc^* \leq T(\eE Q_i) \leq \eE Q_{i+1} \leq \eE Q_i,
\]
assuming that Q-value to be updated in each step is selected independently at random.
\end{lemma}

\begin{proof}
We show this by induction. For the induction basis $(i=0)$, we have that $\bfc^* \leq Q_0$ by definition.

As $\bfc^*$ is the fixed-point of $T$, we have $T(\bfc^*)=\bfc^*$, and the monotonicity of $T$ provides $T(\bfc^*) \leq T(Q_0)$.
At the same time 
\begin{align*}
T(Q_0) = T(\kappa \bfc^*) &= B \kappa \bfc^* + \bfc\\
&= \kappa (B \bfc^* + \bfc) - \kappa \bfc + \bfc\\
&= \kappa \bfc^* - (\kappa -1) \bfc \\
&= Q_0 - (\kappa -1) \bfc \leq Q_0.    
\end{align*}

This provides $\bfc^* \leq T(\eE Q_0) \leq \eE Q_0$. Finally, $ T(\eE Q_0) \leq \eE Q_0$ entails for a learning rate $\lr_0 \in [0,1]$ that $T(\eE Q_0) \leq \eE Q_{1} \leq \eE Q_0$ due to (\ref{eq:exp-target}).

For the induction step ($i \mapsto i+1$), we use the induction hypothesis
\[
\bfc^* \leq T(\eE(Q_i)) \leq \eE(Q_{i+1}) \leq \eE(Q_i).
\]

The monotonicity of $T$ and $\bfc^* \leq \eE(Q_{i+1}) \leq \eE(Q_i)$ imply that $T(\bfc^*) \leq T(\eE(Q_{i+1})) \leq T(\eE(Q_i))$ holds. With $T(\bfc^*) = \bfc^*$ (from the fixed point equations) and the induction hypothesis,
$\bfc^* \leq T(\eE(Q_{i+1})) \leq \eE(Q_{i+1})$ follows.

Using $T(\eE Q_{i+1} )=\eE(T(Q_{i+1}))$, this provides $\eE(T(Q_{i+1})) \leq \eE Q_{i+1}$, which implies with $\lambda_{i+1} \in [0,1]$ that 
\[
T(\eE Q_{i+1})=\eE(T(Q_{i+1}))\leq \eE Q_{i+2} \leq \eE Q_{i+1}
\]
holds, completing the induction step. \qed
\end{proof}

By simply replacing all $\leq$ with $\geq$ in the above proof, we can get the following for all initial Q-values that happen to be $\kappa\bfc^*$ where $\kappa \leq 1$:

\begin{lemma}
\label{lem:BP_low}
Let $Q_0 = \kappa \bfc^*$ for a scalar factor $\kappa \in [0,1]$. Then the following holds for all $i \in \mathbb N$, assuming that the Q-value to update in each step is selected independently at random: $\bfc^* \geq T(\eE Q_i) \geq \eE Q_{i+1} \geq \eE Q_i$. \qed
\end{lemma}

We now first establish that the distance between $Q$ and $\bfc^*$ can be upper bounded by the distance between $Q$ and $T(Q)$ with a fixed linear factor $\mu > 0$. 

\begin{lemma}
\label{lem:fixed-factor}
There exists a constant $\mu > 0$ such that
\[
\sum_{q \in \TT}|(Q - T(Q)(q)| \geq \mu \sum_{q \in \TT}|(Q - \bfc^*(q)|
\]
when $Q_0 = \kappa \bfc^*$.
\end{lemma}
\begin{proof}
We show this for $\kappa > 1$. The proof for $\kappa < 1$ is similar, and there is nothing to show for $\kappa = 1$.

We first consider the linear programme with a variable for each type with the following constraints for some fixed $\delta>0$:
\[
Q \geq \bfc^*, T(Q) \leq Q, \text{ and } \sum_{q \in \TT}Q(q) = \sum_{q \in \TT} \bfc^*(q) + \delta.
\]

An example solution to this constraint system is
$Q = (1 + \frac{\delta}{\sum_{q \in \TT}\bfc^*(q)}) \bfc^*$.

We then find a solution minimising the objective $\sum_{q \in \TT}|(Q - T(Q)(p)|$, noting that all entries are non-negative due to the first constraint.
This is expressed by adding $2|\TT|$ constraints 
\begin{align*}
x_q &\geq Q(q) - T(Q)(q) \\
x_q &\geq T(Q)(q) - Q(q)
\end{align*}
and minimising $\sum_{q \in \TT}x_q$.

As $\bfc^*$ is the only fixed-point of $T$, and $\sum_{q \in \TT}Q(q) = \sum_{q \in \TT} \bfc^*(q) + \delta$ implies that, for an optimal solution $Q^*$, $Q^* \neq \bfc^*$, we have that
\[
\sum_{q \in \TT} |(Q^* - T(Q^*)(q)| > 0.
\]
Due to the constraint $Q \geq \bfc^*$, we always have $Q = \bfc^* + Q_\Delta$ for some
$Q_\Delta > \bfzero$.
We can now re-formulate this linear programme to look for $Q_\Delta$ instead of $Q$:
\begin{align*}
    Q_\Delta &\geq \bfzero,\\
    B Q_\Delta & \leq Q_\Delta, \text{ and } \\
    \sum_{q \in \TT} Q_\Delta(q) &= \delta,
\end{align*} 
with the objective to minimise $\sum_{q \in \TT} | ( Q_\Delta - B Q_\Delta)(q)|$.

The optimal solution $Q_\Delta^*$ to this linear programme gives an optimal value $Q^* = \bfc^*+Q_\Delta^*$ for the former and, vice versa, the value $Q^*$ for the former provides an optimal solution $Q_\Delta^* - \bfc^*$ for the latter, and these two solutions have the same value in their respective objective function.

Thus, while the former constraint system is convenient to show that the value of the objective function is positive, the latter constraint system is, except for $\sum_{q \in \TT} Q_\Delta(q) =\delta$, linear. This means that any optimal solution for $\delta = \delta_1$ can be obtained from the optimal solution for $\delta = \delta_2$ just by rescaling it by $\delta_1/\delta_2$. It follows that the optimal value of the objective function is linear in $\delta$, e.g., there exists $\mu > 0$ such that its value is $\mu \delta$. 
\qed
\end{proof}

We now show that the sequence of Q-values updates converges in expectation to $\bfc^*$
when $Q_0 = \kappa \bfc^*$.

\begin{lemma}
\label{lem:contract_up}
Let $Q_0 = \kappa \bfc^*$ where $\kappa \geq 0$.
Then,  assuming that each type-action pair is selected for update with a minimal probability $p_{\min}$ in each step, and that $\sum_{i=0}^\infty \lr_i = \infty$, then
$\lim_{i \rightarrow \infty} \eE Q_i = \bfc^*$ holds.
\end{lemma}

\begin{proof}
We proof this for $\kappa \geq 1$. A similar proof shows this for any $\kappa \in [0,1]$.
Lemma \ref{lem:BP_up} provides that all $\eE Q_i$ satisfy the constraints $\eE Q_i \geq \bfc^*$ and $T(\eE(Q_i))  \leq \eE Q_i$.

Let $p_{\min}$ be the smallest probability any Q-value is selected with in each update step. Due to Lemma \ref{lem:fixed-factor}, there is a fixed constant $\mu > 0$ such that 
\[\sum_{q \in \TT} |Q_i (q) - T(Q_i)(q)|
\geq \mu \sum_{q \in \TT} |Q_i(q) - \bfc^*(q)| \enspace. \]
By taking the expected value of both sides and the fact that $\bfc^* \leq T(\eE Q_i) \leq \eE Q_{i+1} \leq \eE Q_i$ due to Lemma \ref{lem:BP_up}, we get
\[\sum_{q \in \TT} \eE Q_i (q) - T(\eE Q_i)(q)
\geq \mu \sum_{q \in \TT} \eE Q_i(q) - \bfc^*(q),\] then due to (\ref{eq:exp-target}) we have
\[\sum_{q \in \TT} \eE Q_i(q) - \eE Q_{i+1}(q) \geq \mu p_{\min} \lr_i \sum_{q \in \TT} \eE Q_i(q) - \bfc^*(q),\]
and finally just by rearranging these terms we get
\[\sum_{q \in \TT} \eE Q_{i+1}(q) - \bfc^*(q) \leq (1-\mu p_{\min} \lr_i) \sum_{q \in \TT} \eE Q_i(q) - \bfc^*(q) \enspace.\]

Note that all summands are positive by Lemma \ref{lem:BP_up}.

With $\sum_{i=0}^\infty \lr_i = \infty$, we get that $\sum_{i=0}^\infty \mu p_{\min} \lr_i = \infty$, because $p_{\min}$ and $\mu$ are fixed positive values. This implies that $\prod_{i=0}^{\infty} (1-\mu p_{\min} \lr_i) = 0$ and so the distance between $\eE Q_{i}$ and $\bfc^*$ converges to 0. 
\qed
\end{proof}
Lemma \ref{lem:contract_up} suffices to show convergence of Q-values in expectation.

\begin{theorem}
\label{theo:weak_converge}
When each Q-value is selected for an update with a minimal probability $p_{\min}$ in each step, and $\sum_{i=0}^\infty \lr_i = \infty$, then
$\lim_{i \rightarrow \infty} \eE Q_i = \bfc^*$ holds for every starting Q-values $Q_0 \geq \bfzero$.
\end{theorem}

\begin{proof}
We first note that none of the entries of $\bfc^*$ can be $0$. This implies that there is a scalar factor $\kappa \geq 0$ such that $\bfzero \leq Q_0 \leq \kappa \bfc^*$.
As the $Q_i$ are monotone in the entries of $Q_0$, and as the property holds for $Q_0' = \bfzero = 0 \cdot \bfc^* $ and $Q_0'' = \kappa \bfc^*$ by Lemma \ref{lem:contract_up}, the squeeze theorem implies that it also holds for $Q_0$.
\qed
\end{proof}

Convergence of the expected value is a weaker property than expected convergence, which also explains why our assumptions are weaker than in Theorem \ref{thm:q-learning}. With the common assumption of sufficiently fast falling learning rates, $\sum_{i=0}^\infty {\lr_i}^2 < \infty$, we will now argue that the pointwise limes inferior of the sequence of Q-values almost surely converges to $\bfc^*$. This will later allow us to infer convergence of the actual sequence of Q-values to $\bfc^*$. 

\begin{theorem}
\label{theo:strong}
When each Q-value is selected for update with a minimal probability $p_{\min}$ in each step, 
\[
\sum_{i=0}^\infty \lr_i = \infty \text{ and } \sum_{i=0}^\infty \lr_i^2 < \infty,
\]
then
$\lim_{i \rightarrow \infty} Q_i = \bfc^*$ holds almost surely for every starting Q-values $Q_0 \geq \bfzero$.
\end{theorem}

\begin{proof}
We assume for contradiction that, for some $\widehat{Q} \neq \bfc^*$, there is a non-zero chance of a sequence $\{Q_i\}_{i \in \mathbb N_0}$ such that
\begin{itemize}
\item $\|\widehat{Q}-\liminf_{i \rightarrow \infty} Q_i\|_\infty < \varepsilon'$ for all $\varepsilon' > 0$, and
\item there is a type $q$ such that
$\widehat Q(q) < T(\widehat Q)(q)$.
\end{itemize}
Then there must be an $\varepsilon>0$ such that $\widehat{Q}(q)+3 \varepsilon < T(\widehat{Q} - 2  \varepsilon\cdot\bfone)(q)$.
We fix such an $\varepsilon>0$.

Now we have the assumption that the probability of $\|\widehat{Q}-\liminf_{n \rightarrow \infty} Q_i\|_\infty < \varepsilon$ is positive.
Then, in particular, the chance that, at the same time,
$\liminf_{i \rightarrow \infty} Q_i > \widehat{Q} - \varepsilon\cdot\bfone$ \emph{and} $\liminf_{i \rightarrow \infty} Q_i < \widehat{Q} + \varepsilon \cdot \bfone$, is positive.

Thus, there is a positive chance that the following holds:
there exists an $n_\varepsilon$ such that, for all $i>n_\varepsilon$,  $Q_i \geq \widehat{Q} - 2 \varepsilon\cdot\bfone$.
This implies 
\[
T(Q_i)(q) \geq T(\widehat{Q} - 2 \varepsilon\cdot\bfone)(q) > \widehat Q(q)+ 3 \varepsilon.
\]

Thus, the expected limit value of $Q_i(q)$ is at least $\widehat Q(q)+3\varepsilon$, for every tail of the update sequence.
Now, we can use $\widehat Q - 2\varepsilon$ as a bound on the estimation of the updates in $Q$-learning as
 $Q_i \geq \widehat{Q} - 2 \varepsilon\cdot\bfone$ holds.
At the same time, the variation of the sum of the updates goes to $0$ when $\sum_{i+0}^\infty \lr_i^2$ is bounded.
Therefore, it cannot be that $\liminf_{i \rightarrow \infty} Q_i < \widehat{Q} + \varepsilon \cdot \bfone$ holds; a contradiction.

We note that if, for a Q-values $Q\geq \bfzero$, there is a $q\in \TT$ with $Q(q') < \bfc^*(q')$, then there is a $q \in \TT$ with $Q(q) < T(Q)(q)$ and $Q(q) < \bfc^*(q)$.
This is because, for the Q-values $Q'$ with $Q'(q) = \min\{Q(q),\bfc^*(q)\}$ for all $q\in Q$, $Q' < \bfc^*$. Thus, there must be a type $q\in \TT$ such that $\kappa = \frac{Q'(q)}{\bfc^*(q)} < 1$ is minimal, and $Q' \geq \kappa \bfc^*$.
As we have shown before, $T(\kappa \bfc^*) = \kappa \bfc^* - (\kappa-1) \bfc$, such that the following holds:
\[
T(Q)(q) \geq T(Q')(q) \geq T(\kappa \bfc^*)(q) = \kappa \bfc^*(q) + (1-\kappa) c(q) > \bfc^*(q) = Q(q).
\]
Thus, we have that $\liminf_{i \rightarrow \infty} Q_i \geq \bfc^*$ holds almost surely.
With $\lim_{i \rightarrow \infty} \eE Q_i = \bfc^*$, it follows that
$\lim_{i \rightarrow \infty} Q_i = \bfc^*$.
\qed
\end{proof}

\subsection{Convergence for BMDPs and finite $\bfc^*$}
We start with showing that, for BMDPs, the pointwise limes superior of each sequence is almost surely less than or equal to $\bfc^*$.
We then proceed to show that the limes inferior of a sequence is almost surely $\bfc^*$, which together implies almost sure convergence.

\begin{lemma}
\label{lem:bdpsup}
When each Q-value of BMDP is selected for update with a minimal probability $p_{\min}$ in each step, $\sum_{i=0}^\infty \lr_i = \infty$, $\sum_{i=0}^\infty \lr_i^2 < \infty$, then
$\limsup_{i \rightarrow \infty} Q_i \leq \bfc^*$ holds almost surely for every starting Q-values $Q_0 \geq \bfzero$.
\end{lemma}

\begin{proof}
To show the property for the limes superior, we fix an optimal static strategy $\sigma^*$ that exists due to Corollary \ref{cor:optimal-static}.

We define an BMC obtained by replacing each type $q$ in the BMDP with $A(q) = \{a_1,\ldots,a_k\}$, by $k$ types $(q,a_1), \ldots, (q,a_k)$ with one action,
where each type $q'$ is replaced by the type-action pair $(q',\sigma^*(q'))$.

It is easy to see that a type $(q,\sigma^*(q))$ for the resulting BMC has the same value as the type $q$ and the type-action pair $(q,\sigma^*(q))$ in the BMDP that we started with.

When identifying these corresponding type-action pairs, we can look at the same sampling for the BMDP and the BMC, leading to sequences $Q_0,Q_1, Q_2, \ldots$ and $Q_0',Q_1', Q_2',  \ldots$, respectively, where $Q_0 = Q_0'$.

It is easy to see by induction that $Q_i \leq Q_i'$. Considering that $\{Q_i'\}_{i \in \mathbb N}$ almost surely converges to $\bfc^*$ by Theorem \ref{theo:strong}, we obtain our result.
\qed
\end{proof}

\begin{theorem}
\label{thm:convergence_main}
When each Q-value of an BMDP is selected for update with a minimal probability $p_{\min}$, $\sum_{i=0}^\infty \lr_i = \infty$, $\sum_{i=0}^\infty \lr_i^2 < \infty$, then
$\lim_{i \rightarrow \infty} Q_i = \bfc^*$ holds almost surely for every starting Q-values $Q_0 \geq \bfzero$.
\end{theorem}

\begin{proof}
As a first simple corollary from Lemma \ref{lem:bdpsup}, we get the same result for the limes inferior (as $\liminf \leq \limsup$ must hold).

We now assume for contradiction that, for some vector $\widehat{Q} < \bfc^*$, there is a non-zero chance of a sequence $\{Q_i\}_{i \in \mathbb N}$ such that $\|\widehat{Q}-\liminf_{n \rightarrow \infty} Q_i\|_\infty < \varepsilon'$ for all $\varepsilon' > 0$.

As $\widehat{Q}$ is below the fixed point of $T$, there must be one type-action pair $(q,\sigma^*(q))$ such that $\widehat{Q}(q,\sigma^*(q)) < T(\widehat{Q})(q,\sigma^*(q))$ (cf.\ the proof of Theorem \ref{theo:strong}).
Moreover, there must be an $\varepsilon>0$ such that 
\[
\widehat{Q}(q,\sigma^*(q))+3 \varepsilon < T(\widehat{Q} + 2  \varepsilon\cdot\bfone)(q,\sigma^*(q)).
\]
We fix such an $\varepsilon>0$.

Now we assume that the probability of $\|\widehat{Q}-\liminf_{n \rightarrow \infty} Q_i\|_\infty < \varepsilon$ is positive.
Then the chance that, simultaneously,
$\liminf_{i \rightarrow \infty} Q_i(q,\sigma^*(q)) > \widehat{Q}(q,\sigma^*(q)) - \varepsilon$ and $\liminf_{i \rightarrow \infty} Q_i(q,\sigma^*(q)) < \widehat{Q}(q,\sigma^*(q)) + \varepsilon$, is positive.

Thus, there is a positive chance that the following holds:
there exists an $n_\varepsilon$ such that, for all $i>n_\varepsilon$ we have $Q_i \geq \widehat{Q} - 2 \varepsilon\cdot\bfone$. This entails 
\[
T(Q_i)(q,\sigma^*(q)) \geq T(\widehat{Q} - 2 \varepsilon\cdot\bfone)(q,\sigma^*(q)) > \widehat Q(q,\sigma^*(q))+ 3 \varepsilon.
\]
Thus, the expected limit value of $Q_i(q,\sigma^*(a))$ is at least $\widehat Q(q,\sigma^*(a))+3\varepsilon$, for every tail of the update sequence.
Now, we can use $T(\widehat Q - 2\varepsilon\cdot\bfone)(q,\sigma^*(a))$ as a bound on the estimation of $T(Q)(q,\sigma^*(q))$ during the update of the Q-value of the type-action pair $(q,\sigma^*(q))$.
At the same time, the variation of the sum of the updates goes to $0$ when $\sum_{i=0}^\infty \lr_i^2$ is bounded.
Therefore, it cannot be that $\liminf_{i \rightarrow \infty} Q_i(q,\sigma^*(a)) < \widehat{Q}(q,\sigma^*(a)) + \varepsilon$ holds; a contradiction.
\qed
\end{proof}

\subsection{Divergence}
We now show divergence of $Q(q)$ to $\infty$ when at least one of the entries of $\bfc^*(q)$ is infinite. First due to Theorem \ref{thm:unique-fix} and its proof we have that $\bfc^* = \sum_{i=0}^\infty B^i \bfc$ for some non-negative $B$ and positive $\bfc$. Therefore $\bfc^*$ is monotonic in $B$ for BMCs. Likewise, the value of $\bfc^*$ for a BMDP depends only on the cost function and the expected number of successors of each type spawned:
Two BMDPs with same cost functions and the expected numbers of successors have the same fixed point $\bfc^*$. Thus, if a type $q$ with one action spawns either exactly one $q$ or exactly one $q'$ with a chance of $50\%$ each, or if it spawns $10$ successors of type $q$ and another $10$ or type $q'$ with a chance of $5\%$, while dying without offspring with a chance of $95\%$, both lead to identical matrices $B$ and so the same $\bfc^*$ (though this difference may impact the performance of Q-learning).

Naturally, raising the number of expected number of successors of any type for any type-action pair strictly raises $\bfc^*$, while lowering it reduces $\bfc^*$, and for every set of expected numbers, the value of $\bfc^*$ is either finite or infinite.

Let us consider a set of parameters at the fringe of finite vs.\ infinite $\bfc^*$, and let us choose them pointwise not larger than the parameters from the BMC or BMDP under consideration.
As the fixed point from Section \ref{sec:fixed} is clearly growing continuously in the parameter values, this set of expected successors leads to a $\bfc^*$ which is not finite.

We now look at the family of parameter values that lead to $\alpha \in [0,1[$ times the expected successors from our chosen parameter at the fringe between finite and infinite values, and refer to it as the $\alpha$-BMDP. Let also $\bfc^*_\alpha$ denote the fixed point for the reduced parameters.
As the solution to the fixed point grows continuously, so does $\bfc^*_\alpha$.
Moreover, if $\bfc^*_1 = \lim_{\alpha \rightarrow 1} \bfc^*_\alpha$ was finite, then $\bfc^*$ would be finite as well, because then $\bfc^*_1 = \bfc^*$.

Clearly, for all parameters $\alpha \in [0,1[$, the Q-values of an $\alpha$-BMC or $\alpha$-BMDP converge against $\bfc^*_{\alpha}$.
Thus, the Q-values for the BMC or BMDP we have started with converges against a value, which is at least $\sup_{\alpha \in [0,1[}\bfc^*_\alpha$. As this is not a finite value, Q-learning diverges to $\infty$.

\section{Experimental Results}
\label{sec:experiments}
We implemented the algorithm described in the previous section in the formal reinforcement learning tool \toolname{}~\cite{mungojerrie},
a C++-based tool  which reads BMDPs described in an extension of
the PRISM language \cite{kwiatk11}.
The tool provides an
interface for RL algorithms akin to that of \cite{Brockm16} and invokes a linear
programming tool (GLOP)~\cite{ortools} to compute the optimal expected total cost based on the
optimality equations~(\ref{eq:fixed-eq}).

\subsection{Benchmark Suite}
The BMDPs on which we tested Q-learning are listed in
Table~\ref{tab:experiment}.
For each model, the numbers of types in the BMDP, are given.
Table~\ref{tab:experiment} also shows the total cost (as computed by the LP solver), which has full
access to the BMDP.
This is followed by the estimate of the total cost computed by Q-learning
and the time taken by learning.  The learner has several hyperparameters: $\epsilon$ is the exploration rate, $\alpha$ is the learning rate, and
$\mathrm{tol}$ is the tolerance for $Q$-values to be considered different when selecting an optimal strategy.
Finally, ep-l is the maximum episode length and ep-n is the
number of episodes. The last two columns of Table~\ref{tab:experiment} report the
values of ep-l and ep-n when they deviate from the default values.
All performance data are the averages of three
trials with Q-learning.  Since costs are undiscounted, the value
of a state-action pair computed by Q-learning is a direct estimate of
the optimal total cost from that state when taking that action.

\begin{table}[t]
  \caption{Q-learning results.  The default values of the learner
    hyperparameters are: $\epsilon=0.1$, $\alpha=0.1$,
    tol$=0.01$, ep-l$=30$, and ep-n$=20000$.  Times are in seconds.}
  \label{tab:experiment}
  \centering
 \begin{tabular}[c]{l|ccccccccccc}
    Name & ~types~ & ~optimal cost~ & ~estimated cost~ & ~time (avg.)~  & ~ep-l~ & ~ep-n~ \\\hline
    \texttt{cloud1}     & 3   & 5     & 5.026    & 0.369    &       &     \\
    \texttt{cloud2}   & 4   &  5    & 5.016    &  0.369  &       &     \\
    \texttt{bacteria1}   & 3   & 2.5     & 2.514    & 0.374   &              &       \\
    \texttt{bacteria2}   & 3   & 1.34831     & 1.413    & 0.387  &              &      \\
    \texttt{protein}   & 3   & 6     &  5.067   &  0.372 &              &      \\
    \texttt{frozenSmall}   & 16   & 1.84615  & 1.740    & 2.834  &    100   &      \\
    \texttt{rand68}     & 10   & 150.432     & 154.400    & 0.402    &       &      \\
     \texttt{rand283}     & 9   & 4     &   4  &  0.075   &    &  1000    \\
     \texttt{rand945}     & 19   & 212     & 208.177    & 10.756  &  200   &  40000    \\
     \texttt{rand3242}     & 43   & 4     &   4.372 & 5.960   &  100  &     \\
     \texttt{rand6417}     & 62   & 10     & 10    &   12.498     & 50    &      \\
 \end{tabular}
\end{table}

Models \texttt{cloud1} and \texttt{cloud2} are based on the motivating example given in the
introduction.
Examples \texttt{bacteria1} and \texttt{bacteria2} model the population
dynamics of a family of two bacteria~\cite{trivedi2010timed} subject to two treatments.
The objective is to determine which treatment results in the minimum expected cost to extinction of the bacteria population.
The \texttt{protein} example models a stochastic Petri net description~\cite{MunskyK06}
corresponding to a protein synthesis example with entities corresponding to
active and inactive genes and proteins.
The example \texttt{frozenSmall}~\cite{Brockm16} is similar to classical frozen
lake example, except that one of the holes result in branching the process in two
entities. 
Entities that fall in the target cell become extinct.
The objective is to determine a strategy that results in a minimum number of steps before extinction.
Finally, the remaining $5$ examples are randomly created BMDP instances.

\section{Conclusion}
\label{sec:conclusion}
We study the total reward optimisation problem for branching decision processes with unknown probability distributions, and give the first reinforcement learning algorithm to compute an optimal policy.
Extending Q-learning is hard, even for branching processes, because they lack a central property of the standard convergence proof: as the value range of the Q-table is not a priori bounded for a given starting table $Q_0$, the variation of the disturbance is not bounded.
This looks like a more substantial obstacle than the one Q-learning faces when maximising undiscounted rewards for finite-state MDPs, and it is well known that this defeats Q-learning.
So it is quite surprising that we could not only show that Q-learning works for branching processes, but extend these results to branching decision processes, too. 
Finally, in the previous section, we have demonstrated that our Q-learning algorithm works well on examples of reasonable size even with default hyperparameters, so it is ready to be applied in practice without the need for excessive hyperparameter tuning.

\bibliographystyle{splncs04}
\bibliography{papers}

\begin{thebibliography}{10}
\providecommand{\url}[1]{\texttt{#1}}
\providecommand{\urlprefix}{URL }
\providecommand{\doi}[1]{https://doi.org/#1}

\bibitem{becker1977estimation}
Becker, N.: Estimation for discrete time branching processes with application
  to epidemics. Biometrics pp. 515--522 (1977)

\bibitem{brazdil-branching-networks}
Br{\'a}zdil, T., Kiefer, S.: {Stabilization of Branching Queueing Networks}.
  In: 29th International Symposium on Theoretical Aspects of Computer Science
  (STACS 2012). vol.~14, pp. 507--518 (2012).
  \doi{10.4230/LIPIcs.STACS.2012.507}

\bibitem{Brockm16}
Brockman, G., Cheung, V., Pettersson, L., Schneider, J., Schulman, J., Tang,
  J., Zaremba, W.: {OpenAI Gym}. CoRR  \textbf{abs/1606.01540} (2016)

\bibitem{ChenDK12}
Chen, T., Dr{\"{a}}ger, K., Kiefer, S.: Model checking stochastic branching
  processes. In: Proceedings of MFCS 2012. Lecture Notes in Computer Science,
  vol.~7464, pp. 271--282. Springer (2012). \doi{10.1007/978-3-642-32589-2\_26}

\bibitem{journalsiplEsparzaGK13}
Esparza, J., Gaiser, A., Kiefer, S.: A strongly polynomial algorithm for
  criticality of branching processes and consistency of stochastic context-free
  grammars. Inf. Process. Lett.  \textbf{113}(10-11),  381--385 (2013)

\bibitem{ETESSAMI2018355}
Etessami, K., Stewart, A., Yannakakis, M.: Greatest fixed points of
  probabilistic min/max polynomial equations, and reachability for branching
  markov decision processes. Information and Computation  \textbf{261},
  355--382 (2018). \doi{https://doi.org/10.1016/j.ic.2018.02.013}

\bibitem{EtessamiSY20}
Etessami, K., Stewart, A., Yannakakis, M.: Polynomial time algorithms for
  branching markov decision processes and probabilistic min(max) polynomial
  bellman equations. Math. Oper. Res.  \textbf{45}(1),  34--62 (2020).
  \doi{10.1287/moor.2018.0970}

\bibitem{EtessamiWY19}
Etessami, K., Wojtczak, D., Yannakakis, M.: Recursive stochastic games with
  positive rewards. Theor. Comput. Sci.  \textbf{777},  308--328 (2019).
  \doi{10.1016/j.tcs.2018.12.018}

\bibitem{EY09acm}
Etessami, K., Yannakakis, M.: Recursive {M}arkov chains, stochastic grammars,
  and monotone systems of nonlinear equations. J. ACM  \textbf{56}(1) (2009)

\bibitem{EtessamiY15}
Etessami, K., Yannakakis, M.: Recursive markov decision processes and recursive
  stochastic games. J. {ACM}  \textbf{62}(2),  11:1--11:69 (2015).
  \doi{10.1145/2699431}

\bibitem{even2003learning}
Even-Dar, E., Mansour, Y., Bartlett, P.: Learning rates for q-learning. Journal
  of machine learning Research  \textbf{5}(1) (2003)

\bibitem{HJV05}
Haccou, P., Haccou, P., Jagers, P., Vatutin, V.: Branching processes:
  variation, growth, and extinction of populations. No.~5 in Cambridge Studies
  in Adaptive Dynamics, Cambridge University Press (2005)

\bibitem{Har63}
Harris, T.E.: The Theory of branching processes. Springer, Berlin (1963)

\bibitem{Heyde77}
Heyde, C.C., Seneta, E.: I. J. Bienaym\'e: Statistical Theory Anticipated.
  Springer (1977)

\bibitem{jo1987optimal}
Jo, K.Y.: Optimal control of service in branching exponential queueing
  networks. In: 26th IEEE Conference on Decision and Control. vol.~26, pp.
  1092--1097. IEEE (1987)

\bibitem{kiefer2011probabilistic}
Kiefer, S., Wojtczak, D.: On probabilistic parallel programs with process
  creation and synchronisation. In: Proc. of TACAS. pp. 296--310. Springer
  (2011)

\bibitem{KS47}
Kolmogorov, A.N., Sevastyanov, B.A.: The calculation of final probabilities for
  branching random processes. Doklady Akad. Nauk. U.S.S.R. (N.S.)  \textbf{56},
   783--786 (1947)

\bibitem{kwiatk11}
Kwiatkowska, M., Norman, G., Parker, D.: {PRISM 4.0}: Verification of
  probabilistic real-time systems. In: Computer Aided Verification (CAV). pp.
  585--591 (Jul 2011), lNCS 6806

\bibitem{MunskyK06}
Munsky, B., Khammash, M.: The finite state projection algorithm for the
  solution of the chemical master equation. The Journal of Chemical Physics
  \textbf{124}(4),  044104+ (2006)

\bibitem{nielsen2015markov}
Nielsen, L.R., Kristensen, A.R.: Markov decision processes to model livestock
  systems. In: Handbook of operations research in agriculture and the agri-food
  industry, pp. 419--454. Springer (2015)

\bibitem{mungojerrie}
Perez, M., Somenzi, F., Trivedi, A.: \toolname: Formal reinforcement learning.
  \url{https://plv.colorado.edu/mungojerrie/} (2021), {U}niversity of Colorado
  Boulder

\bibitem{ortools}
Perron, L., Furnon, V.: Or-tools (version 7.2).
  \url{https://developers.google.com/optimization} (Jul 2019), {G}oogle

\bibitem{Pli77}
Pliska, S.R.: Optimization of multitype branching processes. Management Science
   \textbf{23}(2),  117--124 (1976)

\bibitem{Put94}
Puterman, M.L.: {M}arkov Decision Processes: Discrete Stochastic Dynamic
  Programming. Wiley (1994)

\bibitem{rao2008classical}
Rao, A., Bauch, C.T.: Classical galton-watson branching process and
  vaccination. International Journal of Pure and Applied Mathematics
  \textbf{44}(4), ~595 (2008)

\bibitem{RW82}
Rothblum, U.G., Whittle, P.: {Growth Optimality for Branching Markov Decision
  Chains}. Mathematics of Operations Research  \textbf{7}(4),  582--601 (1982)

\bibitem{Sutton18}
Sutton, R.S., Barto, A.G.: Reinforcement Learnging: An Introduction. MIT Press,
  second edn. (2018)

\bibitem{trivedi2010timed}
Trivedi, A., Wojtczak, D.: Timed branching processes. In: 2010 Seventh
  International Conference on the Quantitative Evaluation of Systems. pp.
  219--228. IEEE (2010)

\bibitem{udom2013markov}
Udom, A.U.: A markov decision process approach to optimal control of a
  multi-level hierarchical manpower system. CBN Journal of Applied Statistics
  \textbf{4}(2),  31--49 (2013)

\bibitem{watkins1992q}
Watkins, C.J., Dayan, P.: Q-learning. Machine learning  \textbf{8}(3-4),
  279--292 (1992)

\bibitem{WG74}
Watson, H.W., Galton, F.: On the probability of the extinction of families.
  Jour. Anthrop. Inst.  \textbf{4},  138--144 (1874)

\bibitem{Wojtczak09}
Wojtczak, D.: Recursive probabilistic models : efficient analysis and
  implementation. Ph.D. thesis, University of Edinburgh, {UK} (2009),
  \url{http://hdl.handle.net/1842/3217}

\end{thebibliography}

\end{document}